\documentclass[12pt]{l4dc2021} 

\title[Certifying Incremental Quadratic Constraints for Neural Networks]{Certifying Incremental Quadratic Constraints for Neural Networks via Convex Optimization}
\usepackage{times}
\usepackage{amsfonts}       
\usepackage{amsmath} 
\usepackage{algorithm}
\usepackage[noend]{algpseudocode}
\usepackage{graphicx}

\def\bbeta{\boldsymbol{\beta}}
\def\balpha{\boldsymbol{\alpha}}

\newcommand{\bx}{\mathbf{x}}
\newcommand{\by}{\mathbf{y}}

\DeclareMathOperator*{\argmax}{arg\,max}

\author{%
	\Name{Navid Hashemi} \Email{Navid.Hashemi@utdallas.edu}\\
	\addr Mechanical Engineering, University of Texas at Dallas%
	\AND
	\Name{Justin Ruths} \Email{jruths@utdallas.edu}\\
	\addr Mechanical Engineering, University of Texas at Dallas%
	\AND
	\Name{Mahyar Fazlyab} \Email{mahyarfazlyab@jhu.edu} \\
	\addr Mathematical Institute for Data Science, Johns Hopkins University%
}

\begin{document}
	
	\maketitle
	
	\begin{abstract}%
		Abstracting neural networks with constraints they impose on their inputs and outputs can be very useful in the analysis of neural network classifiers and to derive optimization-based algorithms for certification of stability and robustness of feedback systems involving neural networks. In this paper, we propose a convex program, in the form of a Linear Matrix Inequality (LMI), to certify incremental quadratic constraints on the map of neural networks over a region of interest. These certificates can capture several useful properties such as (local) Lipschitz continuity, one-sided Lipschitz continuity, invertibility, and contraction. We illustrate the utility of our approach in two different settings. First, we develop a semidefinite program to compute guaranteed and sharp upper bounds on the local Lipschitz constant of neural networks and illustrate the results on random networks as well as networks trained on MNIST. Second, we consider a linear time-invariant system in feedback with an approximate model predictive controller parameterized by a neural network. We then turn the stability analysis into a semidefinite feasibility program and estimate an ellipsoidal invariant set for the closed-loop system.
	\end{abstract}
	
	
	\section{Introduction}
	
	Due to their ability to capture complex dependencies, Deep Neural Networks (DNNs) have been tremendously successful at various learning tasks such as image classification and learning-based control. Despite this success, the complex structure of neural networks makes them hard to analyze and therefore, they are often used without formal guarantees. For instance, in response to the fragility of neural networks to uncertainties and adversarial attacks, there has been a growing interest in defining an appropriate notion of robustness and building defenses to improve it. Among several measures of robustness is the Lipschitz constant of neural networks, which by definition quantify the sensitivity of the output of the neural network to input perturbations. Knowing this constant is instrumental in several applications, such as robustness certification of classifiers \cite{weng2018evaluating}, stability and safety analysis of deep reinforcement learning controllers, deriving generalization bounds \cite{bartlett2017spectrally,bolcskei2019optimal,sokolic2017robust,neyshabur2017exploring}, and perception-based robust control \cite{dean2019robust}. However, an accurate estimation of this constant can be quite challenging and has spurred significant interest recently.
	
	
	
	More generally, describing neural networks with constraints they impose on their inputs and outputs (e.g., Lipschitz continuity) can be very useful in the analysis of neural networks and to derive optimization-based algorithms for certification of stability and robustness of neural-network-driven feedback systems. In particular, quadratic constraints can be naturally incorporated into existing methods for analysis and design of feedback systems via matrix inequalities \cite{boyd1994linear}.  Motivated by this vision, in this paper we propose a convex program, in the form of an LMI, to certify a class of quadratic constraints on the map of neural networks over a region of interest. These certificates can capture several useful properties such as local and or one-sided Lipschitz continuity, invertibility, contraction, etc.  We illustrate the utility of our approach in two different settings. First, we develop a semidefinite program (SDP) to compute guaranteed and sharp upper bounds on the local Lipschitz constant of neural networks and illustrate the results on random networks as well as networks trained on MNIST. Comparisons with the existing methods reveal that our method is more accurate and more scalable at the same time. Second, we consider a linear time-invariant system in feedback with a neural network that approximates an explicit model predictive control law and turn the stability analysis into an SDP. More specifically, we compute an ellipsoidal invariant set around the equilibrium point of the closed-loop system.
	\footnote{The code is available at \url{https://github.com/mahyarfazlyab/LipSDP-Local}.}

	\subsection{Related Work}
	
	The use of quadratic constraints has a rich history in robust control and leveraged as a tool to abstract nonlinearities, time variations,  unmodeled dynamics, and uncertain parameters by the constraints they impose on their inputs and outputs \cite{yakubovich1992nonconvex,megretski1997system,zames1966input}. Recently, quadratic constraints have been used and adapted for 
	safety verification of neural networks \cite{raghunathan2018semidefinite,fazlyab2019safety,fazlyab2019probabilistic},  estimation of their Lipschitz constants \cite{fazlyab2019efficient}, Lipschitz constrained training of neural networks \cite{pauli2021training} and reachability and stability analysis of  feedback systems with neural network controllers \cite{hu2020reach,yin2020stability}. 
	
	
	
	\medskip
	
	\noindent \textbf{Lipschitz constant esimtation of neural networks.} The value of computing the local Lipschitz constant of a neural network is underscored by the variety of techniques that have been developed to approach the problem \cite{weng2018towards,avant2020analytical,weng2018evaluating,virmaux2018lipschitz,fazlyab2018analysis,latorre2020lipschitz}. 
	%
	%
	\cite{weng2018towards} provide upper bounds (FastLip) by propagating interval-bounds using linear approximations of each neuron depending on whether they are active, inactive, or both over the local region. In \cite{fazlyab2019efficient} the authors propose an SDP called LipSDP that computes guaranteed upper bounds on the \emph{global} Lipschitz constant of deep neural networks. 
	%
	%
	%
	%
	%
	\cite{latorre2020lipschitz} proposed a polynomial optimization framework to bound the local Lipschitz constant (LiPopt) for sparse networks that employ smooth activation functions. \cite{chen2020semialgebraic} extended this polynomial optimization approach (LipOpt) to handle ReLU networks by defining generalized derivatives using a set of semialgebraic constraints. \cite{jordan2020exactly} presents a Mixed-Integer Programming formulation (LipMIP) that computes the exact local Lipschitz constant of a ReLU network and provides a direct linear relaxation (LipLP). Scalability of these methods or the structure imposed to provide scalability is a consistent challenge that competes with the conservatism of the local Lipschitz upper bounds - MIP optimizations become intractable quickly; the sparsity required for LiPopt and LipOpt enables reducing the size of the corresponding linear program. Many of these methods exploit specific features of the activation functions (differentiability, piece-wise linearity, etc) or network topologies (sparsity, scalar-valued network), but limit the scope of networks for which they apply. LipSDP admits an approach that is broadly applicable to most neuron activation functions and the semi-definite programs provide a computational tractability that is practical.

	\medskip
	
	\textbf{Notation.} We denote the set of real $n$-dimensional vectors by $\mathbb{R}^n$, the set of $m\times n$-dimensional matrices by $\mathbb{R}^{m\times n}$, the set of $m\times n$-dimensional matrices with non-negative components by $\mathbb{R}_{+}^{m\times n}$, and the $n$-dimensional identity matrix by $I_n$. We denote by $\mathbb{S}^{n}$, $\mathbb{S}_{+}^n$, and $\mathbb{S}_{++}^n$ the sets of $n$-by-$n$ symmetric, positive semidefinite, and positive definite matrices, respectively. We denote the $p$-norm ($p \geq 1$) by $\|\cdot\|_p \colon \mathbb{R}^n \to \mathbb{R}_{+}$.  We denote the quadratic norm, induced by $P \in \mathbb{S}_{++}^n$, with $\|x\|_P = \sqrt{x^\top P x}$. 
	We denote the $i$-th unit vector in $\mathbb{R}^n$ by $e_i$. We write $\mathrm{diag}(a_1, ..., a_n)$ for a diagonal matrix whose diagonal entries starting in the upper left corner are $a_1, \cdots, a_n$. We denote the hadamard product between two matrices $A,B$ by $A \circ B$. For square matrices $A_1,\cdots,A_m$, $\mathrm{blkdiag}(A_1,\cdots,A_m)$ is a square matrix whose main-diagonal blocks are $A_1,\cdots,A_m$ and all off-diagonal blocks are zero matrices.
	
	\section{Problem Statement} \label{sec:problem}
	Consider a function $f \colon \mathbb{R}^{n_x} \to \mathbb{R}^{n_f}$ parameterized by a feed-forward neural network of the form 
	\begin{alignat}{2} \label{eq: nn equations}
		x_0 &=x \quad 
		x_{k+1} = \phi_k(W_k x_k + b_k) \quad k=0,\cdots,\ell-1 \quad
		f(x_0) = W_{\ell} x_{\ell} + b_{\ell},
	\end{alignat}
	where $W_k \in \mathbb{R}^{n_{k+1} \times n_k}, \ b_k \in \mathbb{R}^{n_{k+1}}$ are the weight and bias of the $k$-th layer, $n_0=n_x, \ n_{f}=n_{\ell+1}$ and $\phi_k \colon \mathbb{R}^{n_{k+1}} \to \mathbb{R}^{n_{k+1}}$ is the layer of activation functions. We denote by $n = \sum_{k=1}^{\ell} n_k$ the total number of neurons.  
	%
	%
	Given two closed sets $\mathcal{X}_0,\mathcal{Y}_0  \subset \mathbb{R}^{n_0}$ in the input space and a symmetric and indefinite matrix $Q_f \in \mathbb{S}^{n_x+n_f}$, we would like to verify that
	%
	%
	\begin{align} \label{eq: local lip bounds}
		\begin{bmatrix}
			x_0-y_0 \\ f(x_0)-f(y_0)
		\end{bmatrix}^\top Q_f \begin{bmatrix}
			x_0-y_0 \\ f(x_0)-f(y_0)
		\end{bmatrix} \geq 0 \quad \forall x_0,y_0 \in \mathcal{X}_0 \times \mathcal{Y} _0.
	\end{align}
	In this paper, we develop an LMI whose feasibility leads to the  certificate \eqref{eq: local lip bounds} for a given $Q_f$.
	%
	We now provide two concrete applications in which the inequalities of the form \eqref{eq: local lip bounds} appear explicitly.
	%

	%
	
	\subsection{Robustness Certification of Neural Network Classifiers} \label{sec:robustness-certification}
	
	
	
	Consider that $f$ functions as a $n_f$-class classifier in which the data $x\in\mathcal{X}\subset\mathbb{R}^{n_x}$ is assigned the class label $i^*(x)=\argmax_i f_i(x)$, where $f_i$ is the $i$-th component of $f$. The robustness of $f$ can be quantified through the adversarial (worst-case) perturbation of minimum norm that is able to change the assigned class label of the point $x$, i.e., $\epsilon_p^*(x) =\{ \inf_{\epsilon} \|\epsilon\|_p \text{ s.t. } i^*(x+\epsilon)\neq i^*(x)\}$ \cite{fawzi2016robustness,peck2017lower}. One technique to identify $\epsilon_2^*(x)$ is to identify the largest $\ell_2$ ball in the \emph{output} space centered at $f(x)$ that maintains the same classification with radius
	$\rho=\min_{i\neq i^*} \tfrac{1}{\sqrt{2}}\left|(e_{i^*}-e_i)^\top f(x)\right|$ \cite{fazlyab2019efficient}.
	If $f$ is locally Lipschitz with constant $L_{f}$ (in $\ell_2$ norm), then it is possible project the ball with radius $\rho$ back into the input space using the Lipschitz constant. This provides a lower bound on the adversarial perturbation $\epsilon_2^*(x)\geq \rho/L_f$.
	
	%
	Certifying quadratic inequalities in the form of \eqref{eq: local lip bounds} enables this analysis because the local Lipshitz continuity of $f$ on $\mathcal{C}_0 \subset \mathbb{R}^{n_x}$ with Lipschitz constant $L_{f}$ is equivalent to 
	\begin{align} \label{eq: incremental qc nn}
		\begin{bmatrix}
			x_0-y_0 \\ f(x_0)-f(y_0)
		\end{bmatrix}^\top \begin{bmatrix}
			L_f^2 I_{n_x} & 0 \\ 0 & - I_{n_f}
		\end{bmatrix} \begin{bmatrix}
			x_0-y_0 \\ f(x_0)-f(y_0)
		\end{bmatrix} \geq 0 \quad \forall x_0,y_0 \in \mathcal{C}_0 \times \mathcal{C} _0.
	\end{align}
	
	\subsection{Stability of Neural Network Controlled Systems}
	
	Consider an LTI system in feedback with a neural network controller, $x_{+} = A x + B f(x)$. 
	%
	Suppose $x_{\star} \in \mathbb{R}^{n_x}$ is an equilibrium of the closed loop system, that is $x_{\star} = A x_{\star} + B f(x_{\star})$. By defining a quadratic Lyapunov function $V(x) = (x-x_{\star})^\top P (x-x_{\star})$ with $P \in \mathbb{S}_{++}^{n_x}$ (to be determined), local geometric stability of the closed-loop system on $\mathcal{D} \subset \mathbb{R}^{n_x}$ (which contains $x_\star$) is implied by the condition $V(x_{+}) \leq \rho V(x)$ for all $x \in \mathcal{D}$, where $\rho \in (0,1)$ is the convergence rate \cite{haddad2011nonlinear}.  But this condition can be equivalently expressed as 
	\begin{align} \label{eq: closed loop stability qc}
		\begin{bmatrix}
			x-x_{\star} \\ f(x)-f(x_\star)
		\end{bmatrix}^\top \begin{bmatrix}
			A^\top P A - \rho P & PB \\ B^\top P  & B^\top P B
		\end{bmatrix}\begin{bmatrix}
			x-x_{\star} \\ f(x)-f(x_\star)
		\end{bmatrix} \leq 0 \quad \forall (x,x_\star) \in \mathcal{D} \times \{x_\star\}.
	\end{align}
	Again, this is an instance of \eqref{eq: local lip bounds}.

	
	\section{Canonical Representation of Nonlinearities}

	The building block of our method is local incremental quadratic constraints, or $\delta\mathrm{QC}$ for short, which aim to describe nonlinear functions/operators incrementally with respect to two arbitrary inputs. 
	A formal definition is as follows, which is inspired by the definition in \cite{accikmecse2011observers}.
	\begin{definition}[Local Incremental Quadratic Constraint] \label{def: Incremental}\normalfont Let $\mathcal{X},\mathcal{Y} \subseteq \mathbb{R}^n$ be two closed sets. We say the function $\phi \! \colon \! \mathbb{R}^n \to \mathbb{R}^n$ satisfies the local incremental quadratic constraint defined by $(\mathcal{X},\mathcal{Y},\mathcal{Q})$ if for any $Q \in \mathcal{Q} \subset \mathbb{S}^{2n}$ we have 
		\begin{align}
			\begin{bmatrix}
				x-y \\ \phi(x)-\phi(y)
			\end{bmatrix}^\top Q \begin{bmatrix}
				x-y \\ \phi(x)-\phi(y)
			\end{bmatrix} \geq 0 \quad \forall (x,y) \in \mathcal{X} \times \mathcal{Y}.
		\end{align}	
	\end{definition}
	Note that $\mathcal{Q}$ is a convex set of all matrices that characterize $\phi$ incrementally with respect to two arbitrary points $(x,y) \in \mathcal{X}\times \mathcal{Y}$. Further, if $\phi$ satisfies the $\delta\mathrm{QC}$ defined by $(\mathcal{X},\mathcal{Y},\mathcal{Q})$, it also satisfies the $\delta\mathrm{QC}$ defined by $(\bar{\mathcal{X}},\bar{\mathcal{Y}},\mathcal{Q})$ for any non-empty $\bar{\mathcal{X}} \subseteq \mathcal{X}$ and $\bar{\mathcal{Y}} \subseteq \mathcal{Y}$.
	

	
	In the sequel, we elaborate on characterizing activation layers in neural networks via $\delta\mathrm{QC}$s. For simplicity in the exposition, we consider the simpler case of $\mathcal{X}=\mathcal{Y}$. Extensions to the more general case would be similar and hence, we omit them for the sake of space.
	
	\subsection{Smooth Activation Functions} 
	%
	
	We start with describing a single activation function over a bounded interval by $\delta\mathrm{QC}$s. See Figure \ref{fig:local_iqc_activation} for an illustration.
	\begin{figure}[t]
		\centering
		\includegraphics[width=0.8\columnwidth]{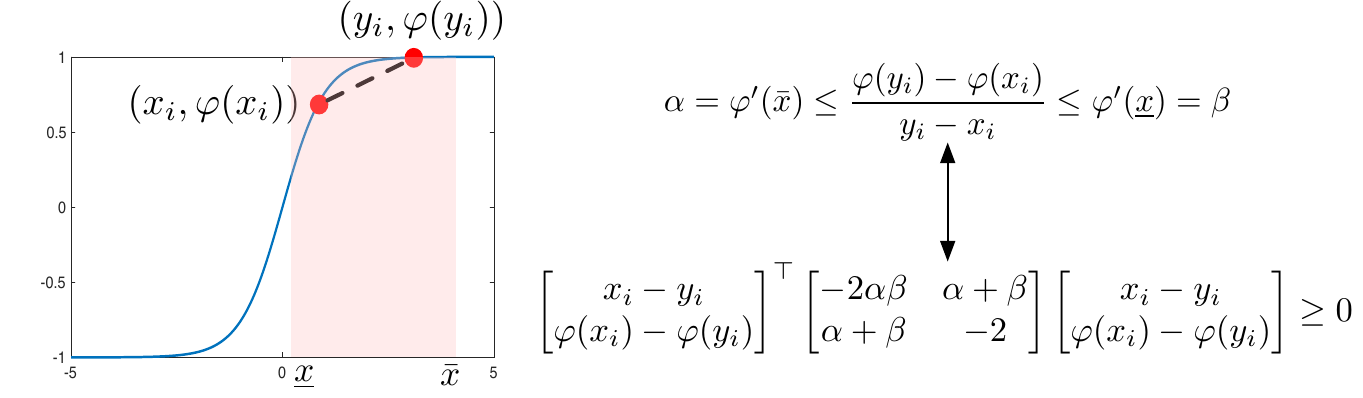}
		\caption{Local incremental quadratic constraint for $\varphi(x)$ on $[\underline{x},\bar{x}]$.}
		\label{fig:local_iqc_activation}
	\end{figure}
	
	\begin{lemma} \label{lemma: local incremental qc single activation}
		Let $\varphi \colon \mathbb{R} \to \mathbb{R}$ be continuous on $\mathcal{X}:=[\underline{x}, \ \bar{x}]$ and  differentiable on $(\underline{x}, \ \bar{x})$. Define $\alpha = \inf_{x \in (\underline{x}, \ \bar{x})} \ \varphi'(x)$ and $\beta = \sup_{x \in (\underline{x}, \ \bar{x})} \ \varphi'(x)$. Then $\varphi$ satisfies the  incremental quadratic constraint defined by $(\mathcal{X},\mathcal{X},\mathcal{Q})$  where
		\begin{align}
			\mathcal{Q} = \{Q \mid Q =   \begin{bmatrix}
				-2\alpha \beta \lambda & (\alpha+\beta) \lambda \\ (\alpha+\beta)\lambda & -2\lambda
			\end{bmatrix}, \ \lambda \geq 0\}.
		\end{align}
		
	\end{lemma}
	%
	Next, we extend the previous lemma to multi-variable nonlinearities.
	\begin{lemma} \label{lemma: local incremental qc multiple activation}
		Let $\phi(x) = (\varphi_1(x_1),\cdots,\varphi_n(x_n)), \ x \in \mathcal{X} \subseteq \mathbb{R}^n$, where all $\varphi_i$'s are differentiable. Define $\alpha_i = \inf_{x \in \mathcal{X}} \ \varphi'(x_i)$ and $\beta_i = \sup_{x \in \mathcal{X}} \ \varphi'(x_i)$. Then $\phi$ satisfies the  $\delta\mathrm{QC}$ defined by $(\mathcal{X},\mathcal{X},\mathcal{Q})$, where
		\begin{align}
			\mathcal{Q} = \{Q \mid Q =  \begin{bmatrix}
				-2\operatorname{diag}(\alpha \circ \beta \circ \lambda ) & \operatorname{diag}((\alpha+\beta) \circ \lambda )  \\ \operatorname{diag}((\alpha+\beta) \circ \lambda )  & -2 \operatorname{diag}(\lambda)
			\end{bmatrix}, \ \lambda \in \mathbb{R}^n_{+} \}.
		\end{align}
	\end{lemma}
	
	\subsection{Piecewise Linear Activation Functions}
	%
	
	Since the Rectified Linear Unit (ReLU) function is not differentiable, the result of Lemma \ref{lemma: local incremental qc single activation} is not directly applicable. Therefore, we characterize (leaky) ReLU functions separately. 
	
	\begin{figure}[t]
		\centering
		\includegraphics[width=0.9\columnwidth]{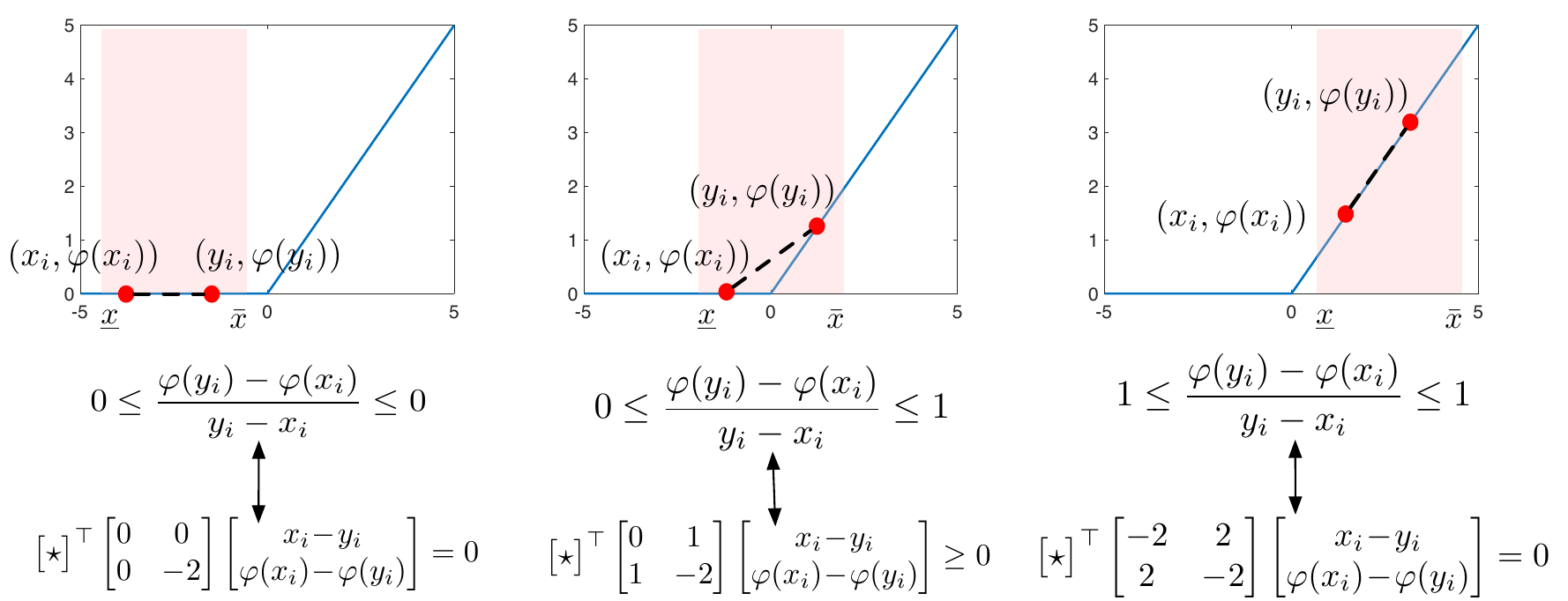}
		\caption{Local incremental quadratic constraint for $\varphi(x)=\max(0,x)$ defined on $[\underline{x},\bar{x}]$.}
		\label{fig:local_iqc_relu}
	\end{figure}
	\begin{lemma}\label{lemma: ReLUQC}
		Let $\phi(x)=\max(\alpha x,\beta x),  \ x \in \mathcal{X} \subseteq \mathbb{R}^n, 0 \leq \alpha \leq \beta <\infty$ and define $\mathcal{I}^{+}$, $\mathcal{I}^{-}$, and $\mathcal{I}^{\pm}$ as the set of activations that are always active, always inactive, or unknown on $\mathcal{X}$, i.e., $\mathcal{I}^{+}= \{i  \mid x_i \geq 0 \text{ for all } x \in \mathcal{X}\}$, $\mathcal{I}^{-}= \{i  \mid x_i < 0 \text{ for all } x \in \mathcal{X}\}$, and $\mathcal{I}^{\pm}= \{1,\cdots,n\} \setminus (\mathcal{I}^{+} \cup \mathcal{I}^{-})$.
		%
		%
		Define $\balpha = [\alpha+(\beta-\alpha)\mathbf{1}_{\mathcal{I}^{+}}(1),\cdots,\alpha+(\beta-\alpha)\mathbf{1}_{\mathcal{I}^{+}}(n)]$ and $\bbeta = [\beta-(\beta-\alpha)\mathbf{1}_{\mathcal{I}^{-}}(1),\cdots,\beta-(\beta-\alpha)\mathbf{1}_{\mathcal{I}^{-}}(n)]$
		%
		%
		Then $\phi$ satisfies the $\delta\mathrm{QC}$ defined by $(\mathcal{X},\mathcal{X},\mathcal{Q})$, where
		\begin{align}
			\mathcal{Q} = \{Q \mid Q =  \begin{bmatrix}
				-2\operatorname{diag}(\balpha \circ \bbeta \circ \lambda ) & \operatorname{diag}((\balpha+\bbeta) \circ \lambda )  \\ \operatorname{diag}((\balpha+\bbeta) \circ \lambda )  & -2\operatorname{diag}(\lambda)
			\end{bmatrix}, \ \lambda_{i} &\in \mathbb{R}_{+} \ \text{ for }  i \in \mathcal{I}^{\pm}\}.
		\end{align}
	\end{lemma}
%
	
	\section{Certifying Local Incremental Quadratic Constraints via Semidefinite Programming}
	
	
	In Theorem \ref{thm: main theorem} we combine the $\delta\mathrm{QC}$s of individual layers to derive an LMI for certifying a $\delta\mathrm{QC}$ for the entire neural network. Before stating the result, we define $\bx = [{x_0}^\top \cdots x_{\ell}^\top]^\top \in \mathbb{R}^{n_0+n}$ and the entry selector matrices $E_k$ such that $x_k = E_k \bx$ for $k=0,\cdots,\ell$. 
	
	\begin{theorem} \label{thm: main theorem} For the neural network in \eqref{eq: nn equations}, define the reachable sets of the pre-activation vectors as
		%
		$\mathcal{D}_{k} = \{W_k x_k + b_k \mid x_i = \phi(W_{i-1} x_{i-1} + b_{i-1}), \ i=1,\cdots,k, \ x_0 \in \mathcal{C}_0\}$, $k=0,\cdots,\ell-1$.
		Suppose each nonlinear layer $\phi_k$ satisfies the local incremental quadratic constraint defined by $(\mathcal{D}_k,\mathcal{D}_k,\mathcal{Q}_k)$. For a given $Q_f \in \mathbb{S}^{n_x+n_f}$, define the following matrix
		\begin{align} \label{eq: thm: main theorem 1}
			M(Q_0,\cdots,Q_{\ell-1},Q_f) = \sum_{k=0}^{\ell-1} \begin{bmatrix}
				W_k E_k \\ E_{k+1}
			\end{bmatrix}^\top Q_k \begin{bmatrix}
				W_k E_k \\ E_{k+1}
			\end{bmatrix} -  \begin{bmatrix}
				E_0 \\ W_{\ell} E_{\ell}
			\end{bmatrix}^\top Q_f \begin{bmatrix}
				E_0 \\ W_{\ell} E_{\ell}
			\end{bmatrix}.
		\end{align}
		If $M(Q_0,\cdots,Q_{\ell-1},Q_f)  \preceq 0$ for some $(Q_0,\cdots,Q_{\ell-1}) \in \mathcal{Q}_0 \times \cdots \times \mathcal{Q}_{\ell-1}$, then 
		\begin{align} \label{eq: thm: main theorem 2}
			\begin{bmatrix}
				x-y \\ f(x)-f(y)
			\end{bmatrix}^\top Q_f \begin{bmatrix}
				x-y \\ f(x)-f(y)
			\end{bmatrix} \geq 0 \quad \forall x,y \in \mathcal{C}_0.
		\end{align}
	\end{theorem}
	%
	%
	%
	
	When specialized to certifying local Lipschitz continuity (see \eqref{eq: incremental qc nn}), Theorem \ref{thm: main theorem} essentially extends the the main result of \cite{fazlyab2019efficient} to local Lipschitz bounds. In \cite{fazlyab2019efficient} it is assumed that the input set $\mathcal{C}_0$ to the neural network is $\mathbb{R}^{n_{x}}$ and hence, the resulting bound is for the global Lipschitz constant. In contrast, we assume a compact input set $\mathcal{C}_0 \subset \mathbb{R}^{n_x}$ and abstract the activation layers over the reachable sets $\mathcal{D}_k$ to find a bound on the local Lipschitz constant. Specifically, the best upper bound on the local Lipschitz constant in $\ell_2$ norm can be obtained by solving the following SDP,
	\begin{alignat}{2}
		&\mathrm{mininimize}  \quad && \rho \\ \notag
		& \text{subject to } \quad && M(Q_0,\cdots,Q_{\ell-1},Q_f) \preceq 0, \\ \notag
		& \quad && Q_k \in \mathcal{Q}_k \quad k=0,\cdots,\ell-1 \quad \rho \geq 0.
	\end{alignat}
	where $Q_f$ is given by $Q_f = \mathrm{blkdiag}(\rho I_{n_x},-I_{n_f})$ and the decision variables are $\rho,Q_0,\cdots,Q_{\ell-1}$. 
	%
	 

	\subsection{Computation of Pre-activation Bounds} \label{subsection: pre-activation bounds}
	
	Instead of finding the reachable sets $\mathcal{D}_k$ exactly, which is computationally prohibitive, we over approximate them by hyper-rectangles, i.e, we seek to find $l^k, u^k \in \mathbb{R}^{n_{k+1}}$, $k=0,1,2,\cdots \ell-1$ such that $l^{k} \leq W_k x_k + b_k \leq u^{k}$. 
	%
	Here we provide an approach based on the idea presented in \cite{zhang2018efficient} to obtain the pre-activation bounds of the current layer given the pre-activation bounds of the previous layer. Specifically, given the reachable set $\mathcal{D}_{k-1}$ we truncate the activation functions on all the neurons and provide two optimal linear functions as the lower and upper bounds. We denote these linear functions by $h_{L,i}^k$, $h_{U,i}^k$, $i=1,2,\cdots,n_k$ which specify the lower and upper bounds on the activation function for the $i$-th neuron located in the $k$-th layer, as depicted in Fig. \ref{fig:uplowlinear}.
	
	\begin{figure}
		\begin{minipage}[t]{0.46\linewidth}
			\vspace{0pt}\centering
			\includegraphics[width=0.8\columnwidth]{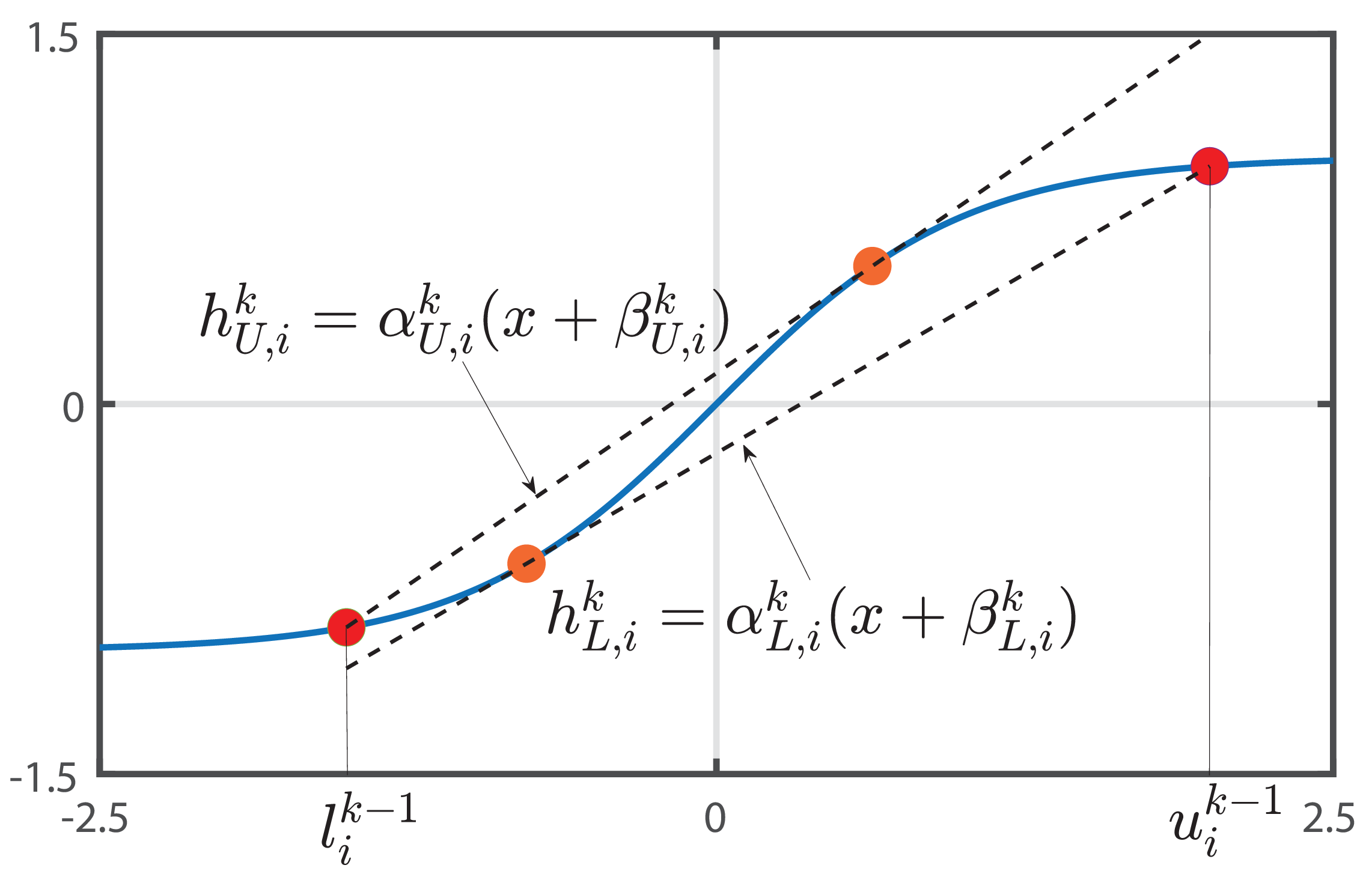}
			\vspace{-0.5cm}
			\caption{\footnotesize{The activation function of neuron $i$ in layer $k$ is lower- and upper-bounded by the linear functions $h_{L,i}^k$ and $h_{U,i}^k$, respectively, over the pre-activation interval $[l_i^{k-1},u_i^{k-1}]$.
			}}
			\label{fig:uplowlinear}
		\end{minipage}
		\hspace{0.03\linewidth}
		\begin{minipage}[t]{0.46\linewidth}
			\vspace{0pt}
			\begin{algorithm}[H]
				\SetKwInput{KwData}{Input} 
				\SetAlgoLined 
				\DontPrintSemicolon
				\KwData{$\mathcal{C}_0:=\left[\underline{x}_0, \overline{x}_0\right]$ (Input the neural network)}
				\KwResult{$\mathcal{D}_k:=\left[l^k,\ u^k\right], k=0,\cdots,\ell-1$.}
				{\small $l^0=-|W_0|\tfrac{\overline{x}_0-\underline{x}_0}{2}+W_0\tfrac{\overline{x}_0+\underline{x}_0}{2}+b_0 $}\\
				{\small$u^0=|W_0|\tfrac{\overline{x}_0-\underline{x}_0}{2}+W_0\tfrac{\overline{x}_0+\underline{x}_0}{2}+b_0 $}\\ 
				\For{$k \in \left\{1,2,\cdots,\ell-1\right\}$}{
					{\small$l^{k}=-|\underline{C}_k|\tfrac{u^{k-1}-l^{k-1}}{2}+\underline{C}_k\tfrac{u^{k-1}+l^{k-1}}{2}+\underline{d}_k$}\\
					{\small $u^{k}=|\overline{C}_k|\tfrac{u^{k-1}-l^{k-1}}{2}+\overline{C}_k\tfrac{u^{k-1}+l^{k-1}}{2}+\overline{d}_k$}}
				\caption{\small Computing pre-activation bounds}
				\label{alg:prebounds}
			\end{algorithm}
		\end{minipage}
	\end{figure}
	
	Table 2 in \cite{zhang2018efficient} provides a comprehensive characterization of the linear functions $h_{L,i}^k(x)=\alpha_{L,i}^k(x+\beta_{L,i}^k)$ and $h_{U,i}^k(x)=\alpha_{U,i}^k(x+\beta_{U,i}^k)$ based on the pre-activation interval $x \in [l_i^{k-1},u_i^{k-1}]$. Concatenating the slopes and intercepts of the $k$-th layer $\alpha_{L,i}^k$, $\alpha_{U,i}^k$, $\beta_{L,i}^k$, and $\beta_{U,i}^k$ into the vectors $\alpha_{L}^k,\alpha_{U}^k,\beta_{L}^k$ and $\beta_{U}^k$, respectively, we then devise Algorithm \ref{alg:prebounds} to iteratively compute the pre-activation bounds over all layers, with the following matrix definitions 
	\begin{equation}\label{eq:matrixdefinition}
		\begin{aligned}
			\underline{C}_k&= W_k^+\operatorname{diag}(\alpha_L^{k})+W_k^-\operatorname{diag}(\alpha_U^{k}),\quad
			\overline{C}_k= W_k^+\operatorname{diag}(\alpha_U^{k})+W_k^-\operatorname{diag}(\alpha_L^{k}),\\
			\underline{d}_k&=W_k^+(\alpha_L^k \circ \beta_L^k)+W_k^-(\alpha_U^k \circ \beta_U^k)+b_k,\quad
			\overline{d}_k=W_k^+(\alpha_U^k \circ \beta_U^k)+W_k^-(\alpha_L^k \circ \beta_L^k)+b_k.
		\end{aligned}
	\end{equation}
	with $W_k^+ = \tfrac{W_k+|W_k|}{2}$ and $W_k^- = \tfrac{W_k-|W_k|}{2}$. 
	We provide a detailed description of this approach in Appendix \ref{apdx:prebound-proof}, including the differences when compared with \cite{zhang2018efficient}.

	\section{Numerical Experiments}
	
	
	

	
	\subsection{Local vs. Global Lipschitz Constants} We consider $\ell_\infty$ ball input sets $\mathcal{C}_0 = \{x_0 \mid \|x_0\|_{\infty} \leq \epsilon\}$ parameterized by $\epsilon$. We then compute the ratio between the local and global Lipschitz bounds (both computed using Theorem \ref{thm: main theorem}) as a function of increasing $\epsilon$. In Figure \ref{fig:Casestudy1}, we plot the results for randomly generated (weights pulled from the normal distribution) hyperbolic tangent neural networks of various depths and widths. 
	
	\begin{figure}[t]
		\centering
		\includegraphics[width=0.9\columnwidth]{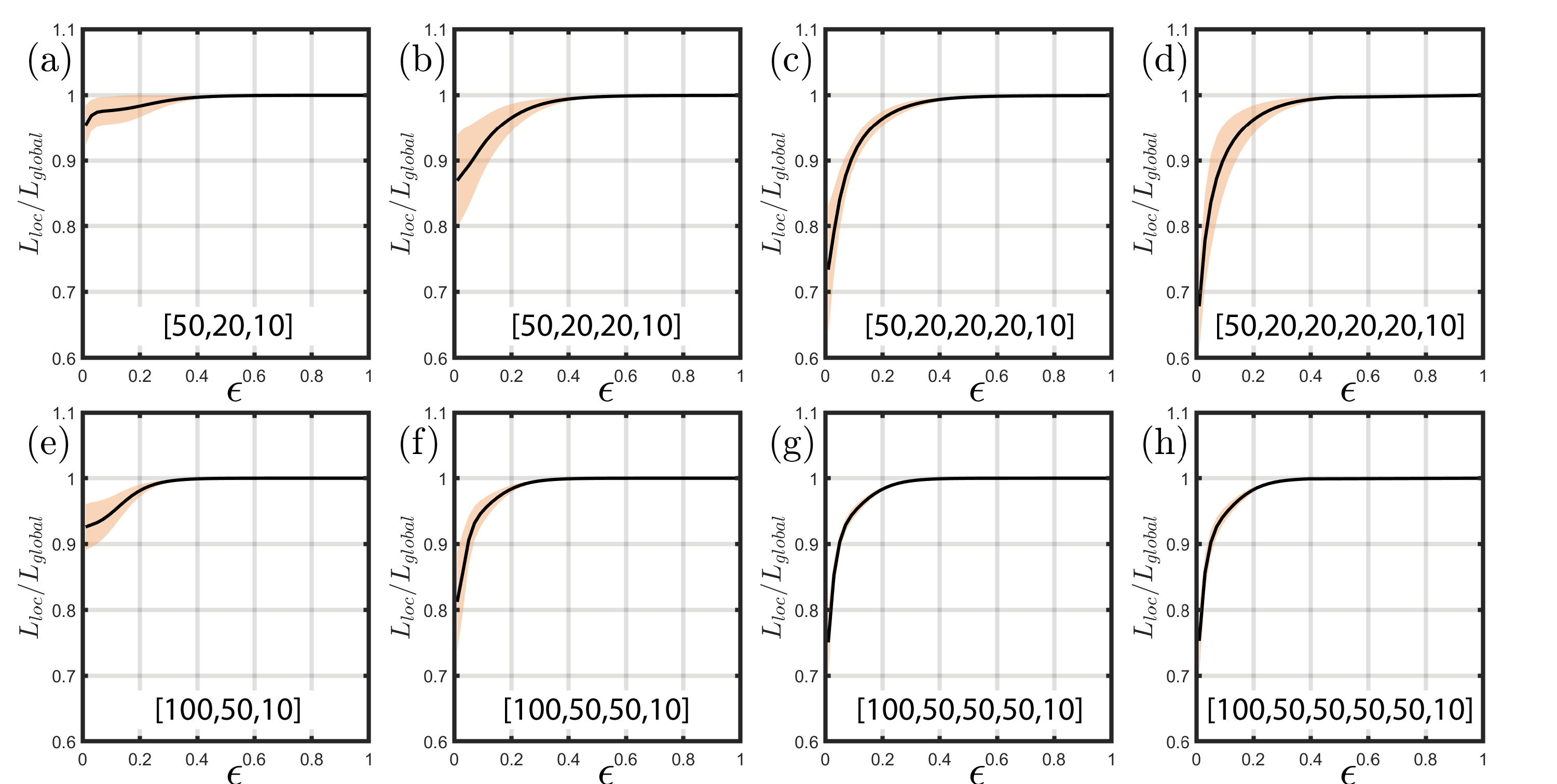}
		\caption{\footnotesize{The local Lipschitz constant approaches the global value as the input set expands in random neural networks with various widths and depths (sizes inset). Black lines demarcate the average and shaded region the standard deviation over 10 different realizations.  
		}}
		\label{fig:Casestudy1}
	\end{figure}
	
	
	
	\medskip
	
	%
	\subsection{MNIST Classification Robustness}
	Following the discussion in Section \ref{sec:problem}, we analyze the robustness of a neural network classifier trained on the MNIST Dataset. 
	Calculating a local Lipschitz constant $L_\text{loc}(\epsilon)$ over a sufficiently large perturbation $\epsilon$, the maximum input perturbation $\epsilon^*$ of the input data $x$ (in $\ell_\infty$ norm) which does not change the classification is given by $\epsilon^*=\frac{\rho}{\sqrt{n_0}L_\text{loc}(\epsilon^*)}$, where the rescaling factor $\sqrt{n_0}$ accommodates changing from the $\ell_2$ norm to the $\ell_\infty$ norm, motivated by pixel-wise perturbations.
	We train a ReLU network using the linear programming-based method of \cite{wong2018provable} (with $\ell_\infty$ perturbation radius $0.05$) with dimensions $[784,100,50,50,10]$. The input datapoint $x$ is labeled as $i^*=1$ and $\rho=11.1427$. Because the local Lipschitz constant is a function of the input perturbation $\epsilon$, we use a bisection algorithm to determine 
	the $\epsilon$ for which $\sqrt{n_0}L_\text{loc}(\epsilon)\epsilon = \rho$. This results in $\epsilon^*=0.05892$ and $L_\text{loc}=6.7529$. If we were to instead use the global Lipschitz constant, $L_\text{global}=14.3764$, the certification radius becomes $\epsilon^*=0.0277$ which is $53\%$ smaller and underscores the enhancement possible using the local Lipschitz constant.
	
%
	\medskip
	
	\noindent \textbf{Comparison with Other Methods. }
	%
	Figure \ref{fig:Comparison} provides a comparison of Local-LipSDP with various current methods for estimating the local Lipschitz constant over an expanding input set. The scalability and technical limitations of some of these methods restrict the sizes of networks we can consider. Nonetheless, the effectiveness of our method is clear even on these relatively simple networks. In Figure \ref{fig:Comparison}a, Local-LipSDP compares well against LiPopt on a random tanh network with layer sizes $[20,20,20,1]$ (LiPopt requires a single output, hence why we compare separately). In Figure \ref{fig:Comparison}b, Local-LipSDP outperforms FastLip and LipLP on a random ReLU network with layer sizes $[2,100,100,2]$. LipMIP provides the exact value at the cost of solving a mixed-integer nonconvex optimization, the scalability of which limits the sizes that can be considered. In this case the sampling approaches (naive lower bound, CLEVER) perform well because the network is small with low-dimension inputs and outputs. The performance of these sample-based methods degrades as the network complexity grows. The experiments in Figure \ref{fig:Comparison} highlight the effectiveness of Local-LipSDP and its capability to handle the full-scale MNIST classifier in the previous experiments demonstrates its scalability and lack of technical limitations compared with other methods.
	

	\begin{figure}[tb]
		\begin{tabular}{ll}
			\begin{minipage}{3in}
				\includegraphics[width=\linewidth]{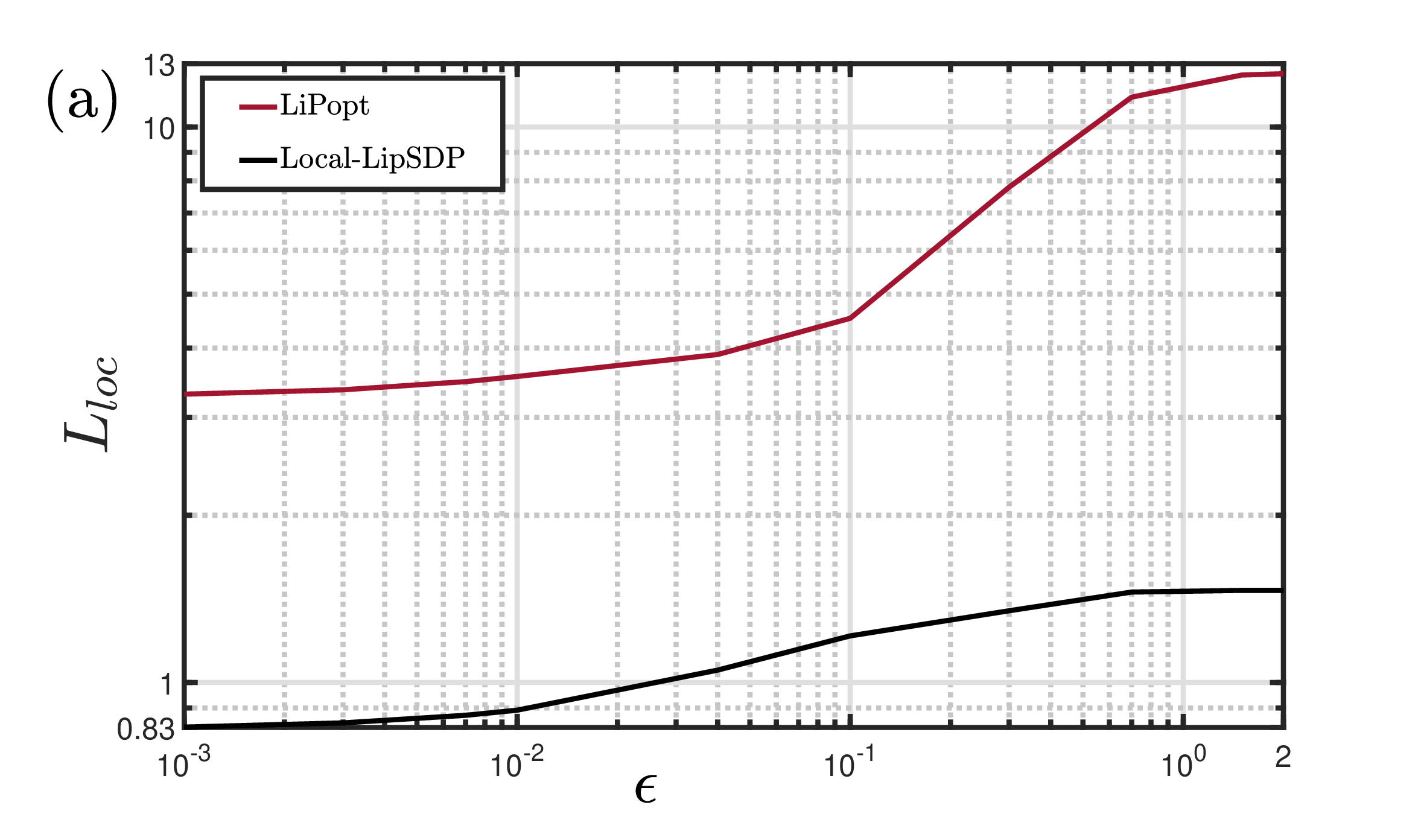}
			\end{minipage}
			&
			\begin{minipage}{3in}
				\includegraphics[width=\linewidth]{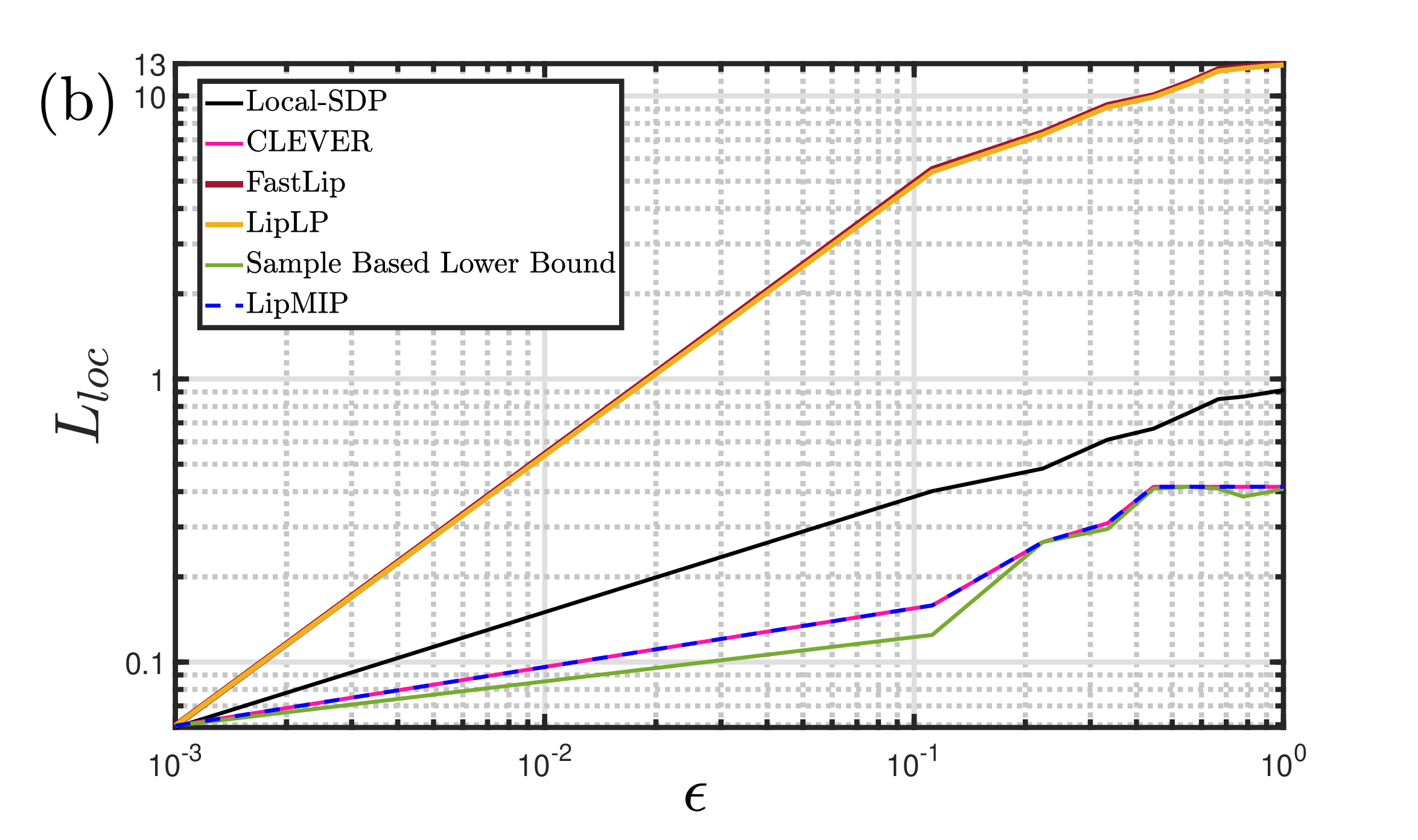}
			\end{minipage}
		\end{tabular}
		\caption{\footnotesize Local-LipSDP outperforms current methods on random tanh network $[20,20,20,1]$ in (a) and ReLU network $[2,100,100,2]$ in (b), chosen relatively small to avoid the scalability challenges and technical limitations of several of the other methods.
		}
		\label{fig:Comparison}
	\end{figure}

	\medskip
	
	\subsection{Stability Analysis of an Approximate MPC Controller}
	Consider a double integrator system
	\begin{align}
		x_{t+1} \!=\! A x_t \!+\! B u_t \quad A = 1.2 \begin{bmatrix}
			1 & 1\\  0 & 1
		\end{bmatrix}, \ B =\begin{bmatrix}
			1 \\ 0.5
		\end{bmatrix}, 
	\end{align}
	subject to the state and input constraints $x_t \in \mathcal{X}=\{ x\mid \|x\|_{\infty} \leq 5\}$ and $u \in \mathcal{U}=\{u \mid \|u\|_{\infty} \leq 1\}$. Consider the finite horizon problem
	\begin{alignat}{2}
		& \mathrm{minimize} \quad &&\sum_{t=0}^{T-1} \|x_t\|_2^2 + u_t^2 \\
		& \text{subject to } \quad && x_{t+1}= Ax_t + Bu_t \notag \\
		& \quad && (x_t,u_t) \in \mathcal{X} \times \mathcal{U} \quad t=0,\cdots,T-1 \notag \\
		& \quad && x_T = 0 \notag
	\end{alignat}
	with the corresponding optimal solution $(x_1^*,\cdots,x_T^*)$ and $(u_0^*,\cdots,u_{T-1}^*)$. THe MPC control law implements $\mu_{MPC}(x)=u_0^\star$. Instead of implementing this control policy, we implement $u_t = \pi(x_t) \approx \mu_{MPC}(x_t)
	$, where $f$ is a ReLU network with architecture $[2,32,32,1]$ that is trained off-line to approximate $\mu_{MPC}(x)$. For generating the training data, we compute $\mu_{MPC}(x)$ at 6284 uniformly chosen random points from the control invariant set. Our goal is to find the largest ellipsoidal invariant set inside the region of interest $\mathcal{D}=\{x \mid \|x\|_{\infty} \leq \epsilon\}$ containing the equilibrium point $x_\star=0$. We first consider the quadratic Lyapunov function $V(x)=x^\top P x$ with $P \succ 0$ and a candidate invariant set $\mathcal{D}_{\beta} =\{x  \mid V(x)\leq \beta \}$.  Now if $V(Ax+B\pi(x)) \leq V(x)$ for all $x \in \mathcal{D}_{\beta}$, then every trajectory starting in $\mathcal{D}_{\beta}$ ($x_0 \in \mathcal{D}_{\beta}$) will remain inside $\mathcal{D}_{\beta}$ ($x_t \in \mathcal{D}_{\beta}$ for all $t \geq 0$). Therefore, the maximum value of $\beta$ such that $\mathcal{D}_{\beta} \subseteq \mathcal{D}$ yields the maximum inner estimate of the positive invariant set. 

	The condition $V(Ax+B\pi(x)) \leq V(x)$ for all $x \in \mathcal{D}_{\beta}$ is equivalent to
	\begin{align} \label{eq: iqc mpc example}
		\begin{bmatrix}
			x-x_{\star} \\ \pi(x)-\pi(x_\star)
		\end{bmatrix}^\top \begin{bmatrix}
			A^\top P A - P & PB \\ B^\top P  & B^\top P B
		\end{bmatrix}\begin{bmatrix}
			x-x_{\star} \\ \pi(x)-\pi(x_\star)
		\end{bmatrix} \leq 0 \quad \forall x \in \mathcal{D}_{\beta}.
	\end{align}
	This is an instance of \eqref{eq: local lip bounds}, where the goal is to find a $P \succ 0$ such that \eqref{eq: iqc mpc example} holds. Note, however, that the input set $ \mathcal{D}_{\beta}$ is a function of the uknown $P$. Therefore, we certify \eqref{eq: iqc mpc example} over the larger set $\mathcal{D}$.  To this end, for a fixed $\epsilon$, we first compute the pre-activation bounds using the method described in $\S$\ref{subsection: pre-activation bounds} using $\mathcal{D}=\{x \mid \|x\|_{\infty} \leq \epsilon\}$ as the input set to the neural network $\pi$. Then, we use Lemma \ref{lemma: ReLUQC} to characterize the activation layers using incremental quadratic constraints. Finally, we invoke Theorem \ref{thm: main theorem}, in which $Q_f$ is given by the negative of the middle matrix in \eqref{eq: iqc mpc example}. If the resulting LMI is feasible $P$, then $\mathcal{D}_{\beta}$ is a positive invariant set.

	To find the largest invariant set, we first use bisection to find the largest $\epsilon$ such that \eqref{eq: iqc mpc example} holds for some positive definite $P$. For this largest value of $\epsilon$, we then find the maximum $\beta$ such that $\mathcal{D}_{\beta} \subseteq \mathcal{D}$, which is equal to $\beta = \min_{i} \epsilon^2 / (e_i^\top P^{-1} e_i)$, where $e_i$ is the $i$-th unit vector in $\mathbb{R}^2$.
	
	For our problem data, we find $\epsilon^\star=0.669$. In Figure \ref{fig:Casestudy2}, we plot the largest ellipsoidal invariant set (in blue) inside $\mathcal{D}$. We also plot the reachable sets $\mathcal{X}_k$ of the closed loop system from the initial set $\mathcal{X}_0 = \mathcal{D}_{\beta}$ to visualize the invariance of $\mathcal{D}_{\beta}$. 
	
	
	\begin{figure}[t]
		\centering
		\includegraphics[width=0.85\columnwidth]{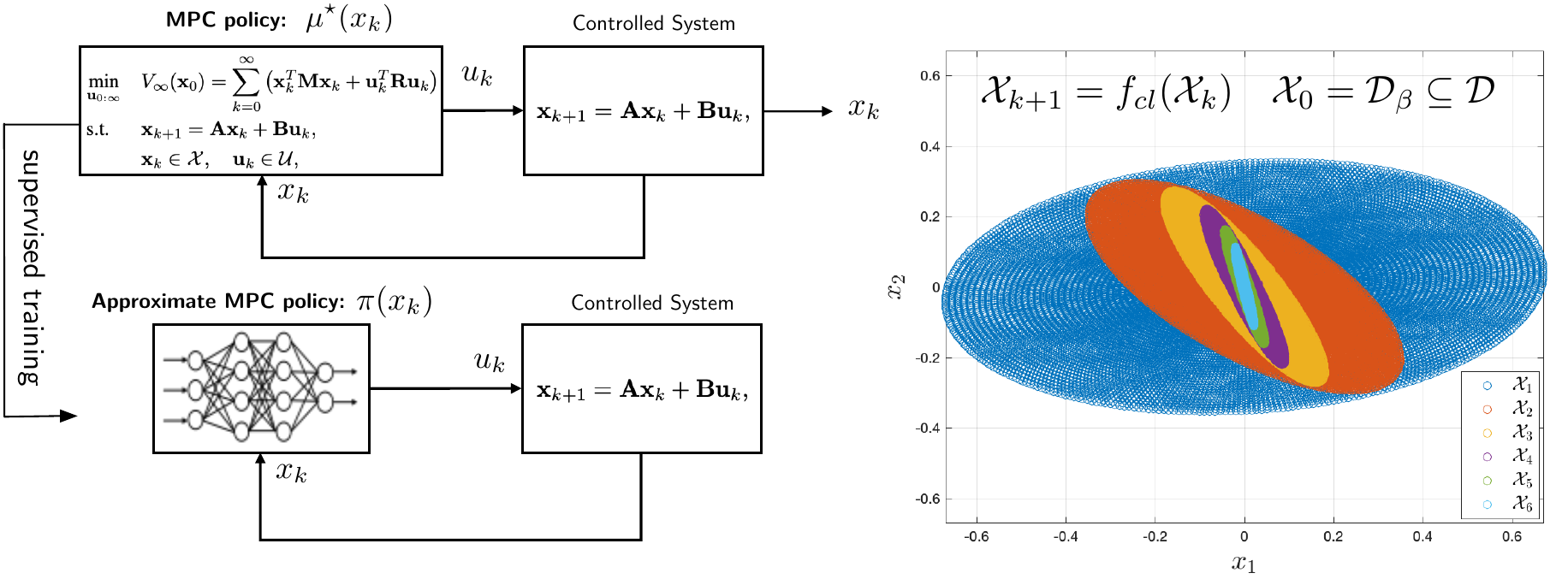}
		\caption{Illustration of an ellipsoidal invariant set computation for the approximate MPC example.}
		\label{fig:Casestudy2}
	\end{figure}
	\section{Conclusion}
	In this paper, we developed a convex program to certify incremental quadratic constraints on the map of neural networks over a region of interest. We illustrated the utility of our method in sharp estimation of local Lipschitz constant of neural networks as well as stability analysis of neural network controlled feedback systems via semidefinite programming. In principle, there is a trade-off between the conservatism, the run time, and the memory requirements of solving these convex programs: certifying tighter bounds typically requires more time and/or memory. Developing numerical algorithms that can span this trade-off is a future research direction.
	\bibliography{Reference}

	\newpage
	
	\appendix

	\section{Proofs} \label{appendix: proofs}

	\subsection{Proof of Lemma \ref{lemma: local incremental qc single activation}} \label{appendix: proof of local incremental qc single activation} 
	
	Using the mean-value theorem on the interval $[x,y] \subseteq [\underline{x},\bar{x}]$, we have $	\varphi'(c) = \frac{\varphi(y)-\varphi(x)}{y-x}$
	for some $c \in (x,y)$.  Therefore, we can write
	%
	$\alpha \leq \inf_{c \in (x,y)} \ \varphi'(c) \leq  \frac{\varphi(y)-\varphi(x)}{y-x} \leq \sup_{c \in (x,y)} \ \varphi'(c) = \beta$.
	These inequalities can be written equivalently as
	$
	\lambda ( \frac{\varphi(y)-\varphi(x)}{y-x} \!-\! \alpha ) ( \frac{\varphi(y)-\varphi(x)}{y-x} \!-\! \beta ) \leq 0
	$, where $\lambda \geq 0$ is arbitrary. Therefore,
	\begin{align*}
		\begin{bmatrix}
			x-y \\ \varphi(x)-\varphi(y)
		\end{bmatrix}^\top  \begin{bmatrix}
			-2\alpha \beta \lambda & (\alpha+\beta) \lambda \\ (\alpha+\beta)\lambda & -2\lambda
		\end{bmatrix} \begin{bmatrix}
			x-y \\ \varphi(x)-\varphi(y)
		\end{bmatrix} \geq 0 \quad \forall x,y \in [\underline{x},\bar{x}].
	\end{align*}	

	\subsection{Proof of Lemma \ref{lemma: local incremental qc multiple activation}} \label{proof of local incremental qc multiple activation}
	By left- and right multiplying $Q$ by $\begin{bmatrix}
		(x-y)^\top & (\phi(x)-\phi(y))^\top
	\end{bmatrix}^\top$ and its transpose, respectively, we obtain
	\begin{align*}
		& \sum_{i=1}^{n} \lambda_{i} \left(-2\alpha_i \beta_i (x_i-y_i)^2+2(\alpha_i+\beta_i)(x_i-y_i)(\varphi_i(x_i)-\varphi_i(y_i))-2(\varphi_i(x_i)-\varphi_i(y_i))^2\right) \nonumber \\ &= \sum_{i=1}^{n} \lambda_{i} \begin{bmatrix}
			x_i-y_i \\ \varphi_i(x_i)-\varphi_i(y_i)
		\end{bmatrix}^\top \begin{bmatrix}
			-2\alpha_i \beta_i & \alpha_i+\beta_i \\ \alpha_i+\beta_i & -2
		\end{bmatrix}\begin{bmatrix}
			x_i-y_i \\ \varphi_i(x_i)-\varphi_i(y_i)
		\end{bmatrix} \geq 0.
	\end{align*}
	where the last inequality follows from Lemma \ref{lemma: local incremental qc single activation}.
	
	\subsection{Proof of Lemma \ref{lemma: ReLUQC}} \label{appendix: proof of RELUQC}
	By left- and right multiplying $Q$ by $\begin{bmatrix}
		(x-y)^\top & (\phi(x)-\phi(y))^\top
	\end{bmatrix}^\top$ and \\ $\begin{bmatrix}
		(x-y)^\top & (\phi(x)-\phi(y))^\top
	\end{bmatrix}^\top$, we obtain
	\begin{align*}
		&\sum_{i=1}^{n} \lambda_{i} \left(-2\alpha_i \beta_i (x_i-y_i)^2+2(\alpha_i+\beta_i)(x_i-y_i)(\varphi_i(x_i)-\varphi_i(y_i))-2(\varphi_i(x_i)-\varphi_i(y_i))^2\right) \nonumber \\
		&= \sum_{i \in \mathcal{I}^{+}} \lambda_i \underbrace{(-2 \alpha_i \beta_i (x_i-y_i)^2 + 2 (\alpha_i+\beta_i)(x_i-y_i)^2-2(x_i-y_i)^2)}_{=0\ \text{( since $\alpha_i=\beta_i=1$)}} \notag \\
		&+ \sum_{i \in \mathcal{I}^{-}} \lambda_i \underbrace{(-2 \alpha_i \beta_i (x_i-y_i)^2 + 2 (\alpha_i+\beta_i)(x_i-y_i)(0-0)-2(0-0)^2)}_{=0 \ \text{( since $\alpha_i=\beta_i=0$)}} \notag \\
		&+ \sum_{i \in \mathcal{I}^{\pm}} \lambda_i \underbrace{\begin{bmatrix}
				x_i-y_i \\ \phi_i(x_i)-\phi_i(y_i)
			\end{bmatrix}^\top \begin{bmatrix}
				-2\alpha_i \beta_i & \alpha_i+\beta_i \\ \alpha_i+\beta_i & -2
			\end{bmatrix}\begin{bmatrix}
				x_i-y_i \\ \phi_i(x_i)-\phi_i(y_i)
		\end{bmatrix}}_{\geq 0},
	\end{align*}
	where the last inequality follows from the fact that $\alpha_i  \leq \frac{\phi_i(x_i)-\phi_i(y_i)}{x_i-y_i} \leq \beta_i$ when $i \in \mathcal{I}^{\pm}$. 
	
	\subsection{Proof of Theorem \ref{thm: main theorem} \label{appendix: proof of main theorem}}

	Suppose $M(\rho,Q_0,\cdots,Q_{\ell-1}) \preceq 0$ for some $(\rho,Q_0,\cdots,Q_{\ell-1}) \in \times \mathcal{Q}_0 \times \cdots \times \mathcal{Q}_{\ell-1}$. By left- and right- multiplying both sides of \eqref{eq: thm: main theorem 1} by $(\bx-\by)^\top$ and $(\bx-\by)$, respectively, we obtain
	\begin{align*}
		\sum_{k=0}^{\ell-1} \begin{bmatrix}
			W_k (x_k-y_k) \\ x_{k+1}-y_{k+1}
		\end{bmatrix}^\top Q_k \begin{bmatrix}
			W_k (x_k-y_k) \\ x_{k+1}-y_{k+1}
		\end{bmatrix} - \begin{bmatrix}
			x_0 - y_0 \\ f(x_0)-f(y_0)
		\end{bmatrix}^\top Q_f \begin{bmatrix}
			x_0 - y_0 \\ f(x_0)-f(y_0)
		\end{bmatrix}\leq 0.
	\end{align*}
	By assumption, for each $k=0,\cdots,\ell-1$, the function $x \mapsto \phi_k \circ (W_k x + b_k)$ satisfies the local incremental quadratic constraint defined by $(\mathcal{D}_k,\mathcal{Q}_k)$. Hence, all summands are non-negative, and as a result,
	\[
	\begin{bmatrix}
		x_0 - y_0 \\ f(x_0)-f(y_0)
	\end{bmatrix}^\top Q_f \begin{bmatrix}
		x_0 - y_0 \\ f(x_0)-f(y_0)
	\end{bmatrix} \geq 0 \quad \forall x_0,y_0 \in \mathcal{C}_0.
	\]

	\section{Details of Algorithm \ref{alg:prebounds}}\label{apdx:prebound-proof}
	By defining $z_k = W_k x_k+b_k$ as the input of the $k$-th hidden layer, the iterations of the neural network can be written as
	\begin{equation*}
		z_{k}=W_k\phi(z_{k-1})+b_k
	\end{equation*}
	Denote $W_k^+=\frac{W_k+|W_k|}{2}$ which contains only the positive elements of $W_k$ and $W_k^-=\frac{W_k-|W_k|}{2}$ where contains the negative elements. Therefore, 
	\begin{equation*}
		z_{k}=W_k^+\phi(z_{k-1})+W_k^-\phi(z_{k-1})+b_k.
	\end{equation*}
	On the other hand,
	\begin{equation*}
		\operatorname{diag}(\alpha_L^k)z_{k-1}+\alpha_L^k \circ \beta_L^k \leq \phi(z_{k-1}) \leq \operatorname{diag}(\alpha_U^k)z_{k-1}+\alpha_U^k \circ \beta_U^k,
	\end{equation*}
	which implies
	\begin{equation*}
		\begin{aligned}
			&W_k^+\left(\operatorname{diag}(\alpha_L^k)z_{k-1}+\alpha_L^k \circ \beta_L^k\right) \leq W_k^+\phi(z_{k-1}) \leq W_k^+\left(\operatorname{diag}(\alpha_U^k)z_{k-1}+\alpha_U^k \circ \beta_U^k\right),\\
			&W_k^-\left(\operatorname{diag}(\alpha_U^k)z_{k-1}+\alpha_U^k \circ \beta_U^k\right) \leq W_k^-\phi(z_{k-1}) \leq W_k^-\left(\operatorname{diag}(\alpha_L^k)z_{k-1}+\alpha_L^k \circ \beta_L^k\right).
		\end{aligned}
	\end{equation*}
	If we define $\underline{C}_k,\underline{d}_k, \overline{C}_k$ and $\overline{d}_k$ as \eqref{eq:matrixdefinition} then we conclude,
	\begin{equation*}
		\underline{C}_k z_{k-1}+\underline{d}_k \leq z_{k} \leq \overline{C}_k z_{k-1}+\overline{d}_k.
	\end{equation*}
	Therefore to find $u^{k}$ and $l^{k}$ we need to solve two different linear programs, which respectively maximize and minimize a linear objective function $f^{\top} z_{k-1}$ such that $z_{k-1}\in [l^{k-1},\ u^{k-1}]$. This linear program has analytical solutions provided in the $\mathbf{for}$ loop embedded inside Algorithm \ref{alg:prebounds}. For the first hidden layer of the neural network, there is no activation function on the input layer so to find the reachable set of inputs in first layer, $\mathcal{D}_0$, we assume linear activation function for input layer, such that, $x_0\in\mathcal{C}_0$ is considered as its output and compute $\mathcal{D}_0$. In this case $\alpha_L^0=\alpha_U^0=\mathbf{1}_{n_0}$ and $\beta_L^0=\beta_U^0=\mathbf{0}_{n_0}$.

\end{document}